\def\year{2022}\relax
\documentclass[letterpaper]{article}
\usepackage{aaai23}  %
\usepackage{times}  %
\usepackage{helvet}  %
\usepackage{courier}  %
\usepackage[hyphens]{url}  %
\usepackage{graphicx} %
\usepackage{natbib}  %
\usepackage{caption} %
\usepackage[american]{babel}

\usepackage{mathtools} %
\usepackage{booktabs} %
\usepackage{tikz} %

\input{macros1.sty}

\title{Assume-Guarantee Reinforcement Learning}

\author {
    Milad Kazemi,\textsuperscript{\rm 1}
    Mateo Perez, \textsuperscript{\rm 2}
    Fabio Somenzi, \textsuperscript{\rm 2}
    Sadegh Soudjani, \textsuperscript{\rm 3}
    Ashutosh Trivedi, \textsuperscript{\rm 2} 
    Alvaro Velasquez \textsuperscript{\rm 2}
}
\affiliations {
    \textsuperscript{\rm 1} King's College London, UK\\
    \textsuperscript{\rm 2} University of Colorado Boulder, USA\\
    \textsuperscript{\rm 3} Max Planck Institute for Software Systems, Germany\\
}

\begin{document}
\maketitle

\begin{abstract}
    We present a modular approach to \emph{reinforcement learning} (RL) in environments consisting of simpler components evolving in parallel.
    A monolithic view of such modular environments may be prohibitively large to learn, or may require unrealizable communication between the components in the form of a centralized controller. 
    Our proposed approach is based on the assume-guarantee paradigm
    where the optimal control for the individual components is synthesized in isolation by making \emph{assumptions} about the behaviors of neighboring components, and providing \emph{guarantees} about their own behavior. 
    We express these \emph{assume-guarantee contracts} as regular languages and provide automatic translations to scalar rewards to be used in RL.
    By combining local probabilities of satisfaction for each component, we provide a lower bound on the probability of satisfaction of the complete system. 
    By solving a Markov game for each component, RL can produce a controller for each component that maximizes this lower bound. The controller utilizes the information it receives through communication, observations, and any knowledge of a coarse model of other agents.
    We experimentally demonstrate the efficiency of the proposed approach on a variety of case studies.
\end{abstract}

\section{Introduction}
\label{sec:introduction}
One approach to synthesize a distributed controller is to first synthesize a centralized controller and then decompose it into a controller for each component. However, the resulting controllers require the full state information of all the components in general, which may be unrealizable. In the case where there is only partial or no communication between each component, producing a distributed controller is undecidable for infinite time horizons and NP-hard for fixed time horizons~\citep{CCT16}. 
\emph{Can we apply reinforcement learning (RL) for distributed policy synthesis without incurring these costs while still providing guarantees on performance?} 
In addressing this question, we propose an {assume-guarantee} approach to RL, where we locally design a controller for each component by abstracting the behavior of neighboring components as an {assume-guarantee} contract. This contract defines a game: the opponent (environment) may produce the worst-case behavior for the neighboring components that still satisfy the assumption, while the controller maximizes the probability of satisfying the guarantee. The probabilities of satisfaction from each game can then be combined to provide a lower bound on the performance of the resulting distributed controller.

\usetikzlibrary{decorations.pathreplacing}

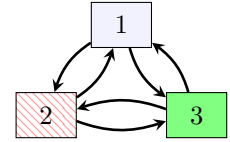
\begin{figure}[t]
    \centering
    \begin{tikzpicture}
    [box/.style={draw,rectangle,minimum width=8mm, minimum height=6mm}]
    \useasboundingbox (-1.5,0.4) rectangle (6,-3);
        \node[] (learn) [] {1) Solve Markov games};
        \node[] (compose) [right=3.75cm of learn] {2) Compose};
        \node[box, fill=blue!5] (mdp) [below left=1.5cm and 0.5cm of learn] {$1$};
        \node [cloud, draw,cloud puffs=11,cloud puff arc=120, aspect=4, minimum height=9mm, pattern=north west lines, pattern color=green] (cloud) [right=1.75cm of mdp] {};
        \node[] [above right=0.1cm and 1.55cm of mdp] {$2$};
        \node[] [below right=0.1cm and 1.9cm of mdp] {$3$};
        \node[] [below=0.75cm of mdp] {$\vdots$};
        \node[] [below=0.75cm of cloud] {$\vdots$};
        \draw [
            thick,
            decoration={
                brace,
                raise=0.6cm
            },
            decorate
        ] (mdp.west) -- (mdp.east)
        node [pos=0.5,anchor=north,yshift=1.2cm] {Guarantee};
        \draw [
            thick,
            decoration={
                brace,
                raise=0.6cm
            },
            decorate
        ] (cloud.west) -- (cloud.east)
        node [pos=0.5,anchor=north,yshift=1.2cm] {Assume};
        \path[->]
        (mdp) edge [bend right, line width=1pt] node {} (cloud)
        (cloud) edge [bend right, line width=1pt] node {} (mdp)
        ;
        \node[box, fill=blue!5] (m1) [above right=0.5cm and 3cm of cloud] {$1$};
        \node[box, fill=green!50] (m2) [below right=1.2cm and 1cm of m1] {$3$};
        \node[box, pattern=north west lines, pattern color=red!40] (m3) [below left=1.2cm and 1cm of m1] {$2$};
        \path[->]
        (m1) edge [bend right=20, line width=1pt] node {} (m2)
        (m1) edge [bend right=20, line width=1pt] node {} (m3)
        (m2) edge [bend right=20, line width=1pt] node {} (m1)
        ([yshift=.8mm]m2.west) edge [bend right=20, line width=1pt] node {} ([yshift=.8mm]m3.east)
        (m3) edge [bend right=20, line width=1pt] node {} (m1)
        ([yshift=-.8mm]m3.east) edge [bend right=20, line width=1pt] node {} ([yshift=-.8mm]m2.west)
        ;
    \end{tikzpicture}
    \label{fig:overview}
    \caption{An overview of assume-gurantee RL. We first form a two player game for each assume-guarantee contract and solve these games with RL. We then compose the resulting controllers and provide a lower bound on its performance.}
\end{figure}

Large environments are often composed of a number of smaller environments with well-defined and often slim interfaces. For instance, consider the traffic intersection signal control problem for a $3\times 3$-grid traffic network shown in Fig.~\ref{fig:traffic-schm} with $9$ intersections each equipped with a traffic light. 
The goal is to design policies to schedule the traffic light signals to keep congestion below some defined threshold. Note that the dynamics for each intersection only depends on the condition of adjacent intersections. It may be desirable to reduce unnecessary communication between the controllers of various intersections; in particular, a centralized control algorithm may be undesirable. Additionally, since the dynamics depend on the flow of traffic, the exact dynamics may not be known and, hence, it is desirable to use RL to design a control policy. 

\begin{figure}
    \centering
    \includegraphics[scale=0.35]{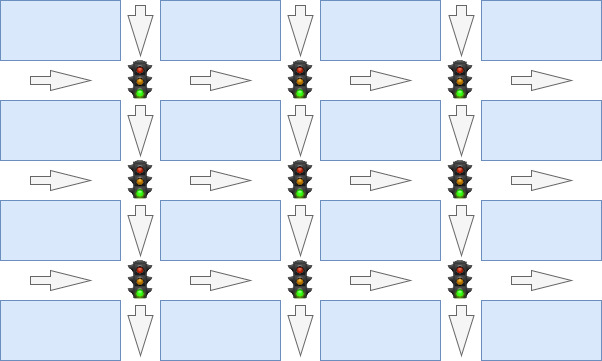}
    \caption{A traffic network with nine intersections.}
    \vspace{-1.5em}
    \label{fig:traffic-schm}
\end{figure}

We work with environments that are naturally decomposable into simpler entities, whose interaction with the rest of the environment is abstracted as an uncontrollable environment.
We apply RL to compute policies for individual components in the environment by making \emph{assumptions} on the behavior of other components in the environment, under which the correctness of the behavior of the individual component is \emph{guaranteed}.
Furthermore, we need to orchestrate the assumptions and guarantees in such a way that the success of RL for individual agents guarantees the success for the network.
We call such learning \emph{assume-guarantee} RL.

A na\"ive application of assume-guarantee reasoning quickly gets circular. 
For instance, if we can learn a controller for all the intersections in the previous example (Fig.~\ref{fig:traffic-schm}) assuming that the controllers in their adjacent intersections satisfy their objectives, does that mean that all intersections satisfy their objectives? 
Of course not: if all of the intersections operate away from the safety zone, the learning requirements on individual agents is vacuously satisfied, and the composed system is not guaranteed to satisfy the objective.
The same reasoning can be extended to more general assumptions.
Mathematical induction provides a way to circumvent this challenge. 
If all controllers begin in their safe zones (base case), and, at every step, assuming that the neighboring components are in the safe state allows each component to guarantee its safety for the next step (inductive step), then by induction it follows that all components in the system will indefinitely remain in their safe zones \citep{McMill98b}.

We generalize the inductive reasoning to handle objectives beyond safety by assuming that we are given qualitative regular specifications as deterministic finite automata (DFA) for all components in the system, and that, every component could receive the exact locations of its neighboring components on their objective DFA at every step either through communication or observation.
Under these assumptions, we derive probabilistic bounds on the global behavior by computing optimal policies for individual components locally with an adversarial view of the environment.
Building on this result, we employ minimax-Q learning~\citep{Littma94,Littma96} to compute policies for individual components providing guarantees on the global behavior. 
The proposed approach gives policies that can also utilize any knowledge of a coarse model of other components.
Our experimental evaluation demonstrates that such assume-guarantee RL is not only helpful in designing controllers with minimal communication, it also scales well due to 1) the reduced state space of individual RL agents and 2) the need to learn policies of homogeneous components only once (i.e., when the components have the same underlying model).

The closest data-driven approach to our problem setting is multi-agent reinforcement learning \citep{bucsoniu2010multi, tan1993multi, castellini2021analysing} that has the following two challenges. First, the multiple objectives in the multi-agent system may be in conflict with each other. Second, having multiple agents learn their policies at the same time may introduce non-stationarity to the system.
Our approach solves these two major challenges by enabling the learning to use coarse abstractions of other agents that are correct over-approximations of their behaviors. These abstractions bound the behavior of other agents and their states are used in the local policies through observations.

\paragraph{Contributions.}
The key contributions of this paper are summarized below:
\begin{enumerate}
    \item We address the problem of satisfying temporal properties on multi-agent systems \citep{ritz2020sat}, where the specification is modeled by a deterministic finite automaton.
    \item We provide a modular data-driven approach adapted to the underlying structure of the system that uses local RL on assume-guarantee contracts and assumes access to the current labels of other neighboring subsystems.
    \item The designed policies are able to use any knowledge of (coarse) abstractions of other agents and observations of states of such abstractions.
    \item We use the dynamic programming characterization of the solution to prove a lower bound for the satisfaction probability of the global specification using local satisfaction probabilities of contracts.
    \item We demonstrate the approach on multiple case studies: a multi-agent grid world, room-temperature control in a building, and traffic-signal control of intersections. 
\end{enumerate}

Due to space limitations, some proofs are presented in the appendix provided as the supplementary material.

\section{Problem Statement}
We write $\R$ and $\N$ to denote the set of real and natural numbers, respectively.
A finite sequence $t$ over a set $S$ is a finite ordered list; similarly, an infinite sequence $t$ is an infinite ordered list.
We write $S^*$ and $S^\omega$ for the set of all finite and infinite sequences over $S$.
For a sequence $t \in S^* {\cup} S^\omega$ we write $t(i)$ for its $i$-th element.
Similarly, for a tuple $t = (a_1, \ldots, a_{k})$ we write $t(i)$ for its $i$-th element.

\subsection{Network of Markov Decision Processes}
For assume-guarantee RL, we consider environments that consist of a network of component environments, each modeled by a finite Markov decision process (MDP).
Throughout, we assume a fixed set of Boolean observations over the state of the MDP called the atomic propositions $AP$.

\begin{definition}[MDPs]
An MDP $M$ is a tuple $(S,s_0, A, \mathbf P,\lab)$ where $S$ is a finite state space, $s_0\in S$ is the initial state, $A$ is the finite set of actions of the controller, $\mathbf P: S {\times} A {\times} S \to [0,1]$ is the transition function, and $\lab: S\to 2^{AP}$ is a state labeling function, mapping states to a subset of the atomic propositions.
The transition function is stochastic, i.e., it satisfies $\sum_{s'\in S} \mathbf P(s,a,s')\in\{0,1\}$ for all $s\in S$ and $a\in A$. 
\end{definition}
We assume the states are \emph{non-blocking}, i.e., there is at least one $a\in A$ for any $s\in S$ such that $\sum_{s'\in S} \mathbf P(s,a,s')=1$.
We denote by $A(s) = \{ a\,|\, \sum_{s'\in S} \mathbf P(s,a,s')=1\}$ the set of actions enabled at state $s \in S$. 
A state trajectory of $M$ is a sequence $s(0), s(1), \ldots \in S^\omega$ such that $s(0)=s_0$, and $s(n+1)\sim \mathbf P(s(n), a(n),\cdot)$ under the action $a(n)\in A(s(n))$ taken at time $n$.  

 \begin{definition}[Markov Games]
A Markov game $G$ is a tuple $(S,s_0, A, B, \mathbf P, \lab)$ where
        $S$ is a finite state space,
        $s_0$ is the initial state,
		$A$ and $B$ are finite sets of actions of the controller and the environment, respectively, 
        $\mathbf P: S \times (A \times B) \times S \to [0,1]$ is the transition function, and
        $\lab: S\to 2^{AP}$ is a state labeling function.
 \end{definition}
The semantics of a Markov game is a concurrent two-player game~\citep{Littma94} between two agents: the controller and its environment.
The state evolution of the Markov game is very similar to an MDP. The only difference is that the probability distribution over the next state depends on the actions chosen concurrently by the controller and the environment from the sets $A$ and $B$, respectively.

\begin{definition}[MDP Network]
    Consider a set $\mathcal{G}=\{G_1, G_2, \ldots, G_n\}$ of components (Markov games) where for every $1 \leq i \leq n$ the component $G_i$ is given by the tuple $(S_i, s_{0i}, A_i, B_i, \mathbf P_i, \lab_i)$.
    An \emph{MDP network} over $\mathcal{G}$ is a graph $(\mathcal{G}, \mathcal{E})$ whose vertices are associated with the components from $\mathcal{G}$ and the set of edges $\mathcal{E} \subset \mathcal{G} \times \mathcal{G}$ (that excludes self-loops)  represents the dependencies between the transition functions of various games. 
    \end{definition}
Given an edge from $G_j$ to $G_i$, the transition function of $G_i$ contains probabilities that depend on the state of $G_j$ modeled as actions chosen by the environment, i.e., for each $G_i$ we have that $B_i := \bigtimes_{j,G_j{\in}\text{Pre}(G_i)}  S_j$, where $\bigtimes$ is the Cartesian product and $\text{Pre}(G_i) = \set{G_j \::\: (G_j,G_i)\in \mathcal{E}}$.

\subsection{Coarse Abstraction of Components}
We incorporate additional knowledge on the dynamic behavior of other components with transition systems defined as Kripke structures. Such a transition system over-approximates all possible behaviors of a component.
	\begin{definition}[Kripke structure]
	    A Kripke structure $M$ over $\alphabet$ is a tuple $(S, I, R, L)$ where $S$ is finite set of states, $I \subseteq S$ is set of initial states, $R \subseteq S \times S$ is transition relation  such that $R$ is left-total, i.e., $\forall s \in S, \exists s' \in S \text{ such that } (s,s') \in R$, and $L: S \rightarrow \alphabet$ is labeling function.   
	\end{definition}

\subsection{Component Specifications}
A deterministic finite automaton (DFA) is a tuple $\DFA = (Q, q_0, F, \alphabet, \sigma)$ where $Q$ is the state space, $q_0$ is the initial state, $F$ is the accepting state, $\alphabet$ is the input alphabet and $\sigma: Q\times \alphabet\to Q$ is the transition function. The transition function $q' = \sigma(q,a)$ specifies the next state $q'\in Q$ from the current state $q\in Q$ under $a\in\alphabet$.
The language $\mathcal L(\DFA)$ of $\DFA$ is the set of sequences $a_0a_1a_2\ldots\in \alphabet^\ast$ such that in the sequence $q_0q_1q_2\ldots$ with $q_n = \sigma(q_{n-1},a_{n-1})$, $n\in\mathbb N$, we have $q_i\in F$ for some $i\in\mathbb N$.

Given two DFAs $\DFA_1 = (Q^1, q^1_{0}, F^1, \alphabet^1, \sigma^1)$ and $\DFA_2 = (Q^2, q^2_0, F^2, \alphabet^2, \sigma^2)$, we define their product DFA $\DFA_1{\times} \DFA_2$ as the tuple $ (Q^\times, q_0^\times, F^\times, \alphabet^\times, \sigma^\times)$, where 
$Q^\times {=} Q^1 {\times} Q^2$, $q_0^\times {=} (q^1_0, q^2_0)$, $F^\times {=} F^1 {\times} F^2$, $\alphabet^\times {=} \alphabet^1 {\times} \alphabet^2$, $\sigma^\times((q_1, q_2), (a_1, a_2)) = (\sigma^1(q_1, a_1), \sigma^2(q_2, a_2))$.
The accepting language of $\DFA_1\times \DFA_2$ is the product of $\mathcal L(\DFA_1)$ and $\mathcal L(\DFA_2)$.
In other words, the language accepted by
the product is the subset of $\alphabet^\times$ such that the projection of each element onto $\alphabet^1$ and $\alphabet^2$ is accepted by the respective automaton.
This construction can be extended to define the accepting language of multiple DFAs. 

We often use linear temporal logic (LTL) notation to succinctly express regular specifications. In particular, we write $\lozenge^n F$ to express bounded reachability property that within a finite bound $n$ a set satisfying $F$ is visited. Similarly, we write $\square^n F$ to express bounded invariance property that no state violating $F$ is visited within $n$ steps.
We write $\lozenge F$ and $\square F$ to capture unbounded reachability and invariance, respectively.

\begin{figure}[t]
    \centering
    \tikzstyle{block} = [draw=black, thick, text width=.6cm, minimum height=.5cm, align=center]  
    \tikzstyle{arrow} = [thick,->,>=stealth]
    \tikzstyle{input} = [coordinate]
    \scalebox{0.8}{
\begin{tikzpicture}
		\Large
		\node[block] (a) {$G_2$};
  \node[block, left=2cm of a.center, anchor=center] (f) {$G_1$};
		\node[input, above left=15mm and 12mm of f] (i) {I};
		\node[input, above left=2 mm and 12mm of f] (j) {J};
    \node[input, below left=2 mm and 12mm of f] (j1) {J1};
    \node[input, below left=2 mm and 12mm of a] (j2) {J2};

    \draw [arrow, rounded corners=2] (a) -- node[pos=0.5,above]{$s_2$}($(a)+(1.2,0)$) |- (i)--(j)--node[pos=0.5,above]{$s_2$}($(f.west)+(0,+.2)$);
    \draw [arrow] ($(f.east)+(0,+.2)$) --node[pos=0.5,above]{$s_1$} ($(a.west)+(0,+.2)$);
    \draw [arrow] (j1) --node[pos=0.5,below]{$a_1$} ($(f.west)+(0,-.2)$);
    \draw [arrow] (j2) --node[pos=0.5,below]{$a_2$} ($(a.west)+(0,-.2)$);

	\end{tikzpicture}
}
    \caption{Feedback composition of two Markov games $G_1$ and $G_2$ in an MDP network.}
    \label{fig:MDPs_feedback}
\end{figure}
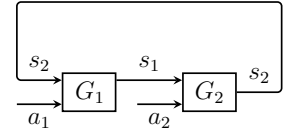

\subsection{Problem Definition}
For ease of presentation, we focus on the network of two components connected in a feedback loop as shown in Fig.~\ref{fig:MDPs_feedback}.
Our proofs can be extended to general networks in a straightforward fashion. 
We are given an MDP network $({\mathcal G}, {\mathcal E})$ with ${\mathcal G} = (G_1, G_2)$ and components $G_i = (S_i, s_{0i}, A_i, B_i, \mathbf P_i, \lab_i)$. 
We are also given a local specification $\phi_i$, modeled as a DFA $\DFA_i = (Q^i, q_{0}^i, F^i, \alphabet^i, \sigma^i)$, for each component $G_i$ for $i = \{1, 2\}$.
Without loss of generality, we assume that the states $F^1$ and $F^2$ of $\DFA_1$ and $\DFA_2$, are absorbing, thus we can interpret the global specification $\phi = \phi_1 \wedge \phi_2$ as reaching the accepting state $F^\times {=} F^1 {\times} F^2$ in the product $\DFA = \DFA_1\times \DFA_2$.

A policy is a recipe to select actions.
We represent the policies of both controllers as a pair of policies.
In particular, we distinguish between the following three classes of policies:
\begin{itemize}
\item {\bf Policies with full state information.} The set $\Pi_f$ of policies with full state information  consists of policy pairs $(\mu_1,\mu_2)$ having the form $\mu_i = (\mu_{i1},\mu_{i2},\mu_{i3}\ldots)$ where $\mu_{in}: S_1{\times} S_2 \to A_i$ selects the input actions at time $n$ for component $i$ based on the full state information of both components.

\item {\bf Policies with no communication.}
The set $\Pi_n$ of policies with no communication consists of policy pairs $(\mu_1,\mu_2)$ having the form $\mu_i = (\mu_{i1},\mu_{i2},\mu_{i3},\ldots)$ where $\mu_{in}:S_i \to A_i$, for $i\in\{1,2\}$ and $n\geq 0$, selects the input action for component $i$ based only on its local state.

\item {\bf Policies with limited communication.}
The set $\Pi_l$ of policies with limited communication consists of policy pairs $(\mu_1,\mu_2)\in\Pi_l$ having the form $\mu_i = (\mu_{i1},\mu_{i2},\mu_{i3},\ldots)$ with  $\mu_{in}: S_i\times Q_{3-i}\to A_i$, with $i\in\{1,2\}$, selects the input action for component $i$ based on the specification automaton state of the other component $3{-}i$.
\end{itemize}
Note that these definitions allow policies to be time-dependent: $\mu_{in}$ for selecting the input action of $G_i$ depends also on the time index $n$. This general form is needed to capture optimal policies for finite-horizon specifications. For infinite-horizon specifications, the policies can be chosen to be stationary (time-independent), and can also be chosen as limits of time-dependent policies. 

\begin{lemma}
We have that $\Pi_n\subset \Pi_l \subset \Pi_f$. Therefore, the following inequality holds for any global specification $\phi$ defined on the joint state evolution of the network:
\begin{equation}
    \sup_{\mu\in\Pi_n}\mathbb{P}^{\mu}(\phi)\le \sup_{\mu\in\Pi_l}\mathbb{P}^{\mu}(\phi)\le \sup_{\mu\in\Pi_f}\mathbb{P}^{\mu}(\phi),
\end{equation}
where $\mathbb{P}^{\mu}(\phi)$ is the probability of satisfying $\phi$ when policy $\mu$ is implemented in the MDP network.
\end{lemma}

\begin{theorem}[Policies and Complexity~\citep{CCT16}]
    The optimal policy with full state information can be computed in polynomial time, while the problem of computing optimal policy with limited or no observation is undecidable for infinite horizon objectives and $\textsc{NP}$-hard for fixed-horizon problems.
\end{theorem}

\begin{assumption}
\label{ass:obs}
For taking action at each time step $n$, the component $G_i$ of the Markov game $\mathcal{G}=\{G_1, G_2, \ldots, G_n\}$ do not have access to the exact state trajectories $s_j(0),s_j(1),\ldots,s_j(n)$ of other components, but it has knowledge of all labeling functions $\lab_j:S_j\rightarrow \alphabet^j$ and the labels $\lab_j(s_j(0)),\ldots,\lab_j(s_j(n))$ along the trajectories of the neighboring components $G_j\in B_i$ either through communication or observation.
Component have policies under which the local properties are satisfied with positive probability. 
Moreover, the components are aware of coarse abstractions of their neighboring components modeled with Kripke structures $M_i = (S_i, I_i, R_i, L_i)$, and observe the transitions in these abstract models.
\end{assumption}
For example, in robotic applications, this assumption could be seen as knowing the mission of other robots and observing which part of the environment they move to in each time step, but not knowing their exact location, their velocity or other physical variables.

A policy that maximizes the satisfaction probability $\mathbb{P}^{\mu}(\phi)$ will be in general a member of $\Pi_f$ that requires the knowledge of the state of the neighboring Markov games. Under Assumption~\ref{ass:obs}, we have restricted the class of policies to have only knowledge of the labels of their neighbors. This limited observation necessitates finding policies for partially observed MDPs (POMDPs).

Computing exact optimal policies on POMDPs requires constructing a belief state, which allows a POMDP to be formulated as an MDP that models the evolution of the belief state~\citep{kaelbling1998planning}. The resulting belief MDP is defined on a continuous state space (even if the original POMDP has a finite state space). This makes the exact solution of the problem computationally intractable \citep{kaelbling1998planning,littman1995efficient}. 
We develop an approximation method by considering the worst-case scenario of the neighboring components under the assumption that we have unlimited control for simulating each MDP in the network.

\begin{problem}[Compositional Synthesis]
\label{prob:lower_extra}
Given components $G_i = (S_i, s_{0i}, A_i, B_i, \mathbf P_i, \lab_i)$, their local specifications $\phi_i$ modeled as DFA $\DFA_i = (Q^i, q_{0}^i, F^i, \alphabet^i, \sigma^i)$, their coarse abstractions modeled with a Kripke structure $M_i = (S_i, I_i, R_i, L_i)$ for $i \in \{1, 2\}$, and a bound $p_\text{min}$, find policies $(\mu_1,\mu_2)\in\Pi_l$ under Assumption~\ref{ass:obs} such that  \[
\mathbb{P}^{\mu_1, \mu_2}(\phi_1\wedge \phi_2) \geq p_\text{min}.
\]
\end{problem}

\begin{remark}
Our first observation is that for the theoretical results, one can take the product of the DFA $\mathcal A_j$ and the Kripke structure $M_j$ and use the product state for the policy synthesis. Note that the non-determinism in the product originates from the coarse abstraction and will be resolved adversarially by taking the worst-case over the non-deterministic transitions. Therefore, we consider only DFA $\mathcal A_i$ without $M_i$ for derivations in the next section.
\end{remark}

Closed-form models of systems are often unavailable or too complex to reason about directly. Model-free reinforcement learning~\citep{Sutton18} is a sampling-based approach to synthesize controllers that compute the optimal policies without constructing a full model of the system, and hence are asymptotically more space-efficient than model-based approaches.
To solve both of aforementioned problems, we develop a sound assume-guarantee based RL algorithm for the compositional synthesis problem.
\blue{}
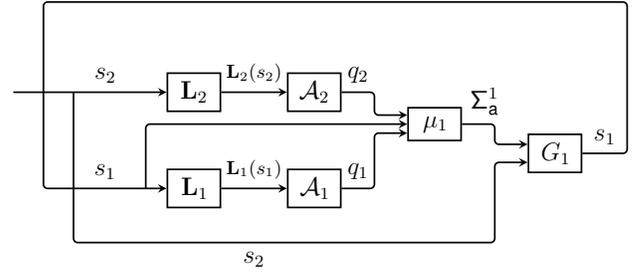
\begin{figure}
    \centering
\tikzstyle{block} = [draw=black, thick, text width=.6cm, minimum height=.5cm, align=center]  
\tikzstyle{arrow} = [thick,->,>=stealth]
\tikzstyle{input} = [coordinate]
\scalebox{0.8}{
\begin{tikzpicture}
		\Large
		\node[block] (a) {$\DFA_2$};
		\node[block, left=2cm of a.center, anchor=center] (g) {$\lab_2$};
		\node[block, below=16mm of a] (b) {$\DFA_1$};
		\node[block, left=2cm of b.center, anchor=center] (o) {$\lab_1$};
		\node[input, left=30mm of g] (y) {Y};
		\node[input, right=10mm of y] (yu) {Yu};
		\node[input, below=25mm of g] (c) {C};
		\node[input, below=4mm of c] (l) {L};
		\node[input, above left=15mm and 25mm of g] (i) {I};
		\node[input, below=31mm of i] (j) {J};
		\node[input, left=8mm of o] (k) {K};

		\node at ($(a.south)+(2,-.2)$) [block] (e) {$\mu_1$};
		\node at ($(e.south)+(2,-.2)$) [block] (f) {$G_1$};

		\draw [arrow, rounded corners=2] (a) -- node[pos=0.5,above] {$q_2$}($(a)+(1,0)$) |- ($(e.west)+(0,.15)$);
		\draw [arrow, rounded corners=2] (b) -- node[pos=0.5,above] {$q_1$}($(b)+(1,0)$) |- ($(e.west)+(0,-.15)$);
		\draw [arrow, rounded corners=2] (yu) |- (c) -- node[pos=0.2,below] {$s_2$}($(c)+(5,0)$) |- ($(f.west)+(0,-.15)$);

		\draw [arrow, rounded corners=2] (e) -- node[pos=0.5,above] {$~~\alphabet^1$}($(e)+(1,0)$) |- ($(f.west)+(0,.15)$);
		\draw [arrow, rounded corners=2] (f) -- node[pos=0.5,above]{$s_1$}($(f)+(1.2,0)$) |- (i)--(j)--node[pos=0.5,above]{$s_1$}(o);
		\draw [arrow, rounded corners=2] (k) |- (e);
		\draw [arrow] (g) --node[pos=0.5,above]{{\small$\lab_2(s_2)$}} (a);
		\draw [arrow] (o) --node[pos=0.5,above]{{\small$\lab_1(s_1)$}} (b);
		\draw [arrow] (y) --node[pos=0.6,above]{{$s_2$}} (g);

	\end{tikzpicture}
}
    \caption{The control structure for $G_1$. The state $s_2$ automatically affects the evolution of $G_1$. The policy uses the state of both automata $\DFA_1,\DFA_2$.
    }
    \label{fig:my_label2}
\end{figure}

\section{Assume-Guarantee RL}
 
We start by presenting the dynamic programming formulation of the probability of satisfying the global specification $\phi = \phi_1\wedge \phi_2$. When considering the joint state evolution of the Markov games and their DFAs $\{(s_i(n),q_i(n)),n \geq 0\}$, this probability is
 \begin{align}\notag
    \mathbb{P}^{\mu_1, \mu_2}&(\phi_1\wedge \phi_2) = \mathbb{E}^{\mu_1, \mu_2}\left[\exists_n\left(q_1(n){\in} F_1 {\wedge} q_2(n){\in} F_2\right)\right]\\
    & = \mathbb{E}^{\mu_1, \mu_2}\left[\sum_{n=0}^{\infty}\mathbf{1}\left(q_1(n){\in} F_1 {\wedge} q_2(n){\in} F_2\right)\right].
 \end{align}
The infinite sum on the right-hand side can be interpreted as the limit of the partial sum whose expectation can be denoted by $v_n = \mathbb{E}^{\mu_1, \mu_2}\left[\sum_{k=0}^{n}\mathbf{1}\left(q_1(k){\in} F_1 {\wedge} q_2(k){\in} F_2\right)\right]$ with $n \geq 0$ such that 
 \begin{equation}
    \mathbb{P}^{\mu_1, \mu_2}(\phi_1\wedge \phi_2)= \lim_{n\to \infty} v_n.
 \end{equation}
The values $v_n$ are generally computed recursively as a function of states $(s_{1},s_{2},q_{1},q_{2})$. To get $\mathbb{P}(\phi_1\wedge \phi_2)$ for an initial state $\overline{s}_0 = (s_{01},s_{02})$,  we need to evaluate the right-hand side at $(s_{01},s_{02},q_{01},q_{02})$. We can use dynamic programming to characterize $v_n$,
 which gives $v_0(s_1, s_2, q_1, q_2)=
    \mathbf{1} (q_1\in F_1) \mathbf{1} (q_2\in F_2)$, and 
    $v_{n+1}(s_1, s_2, q_1, q_2)$ equals $1$ if 
    $(q_1,q_2)\in F_1\times F_2$ and equals  
    \begin{equation}
        \label{eq:dp}
    \mathbb E^{\mu_1,\mu_2}\left[ v_n(\bar s_1, \bar s_2, \bar q_1, \bar q_2)\,|\, s_1, s_2, q_1, q_2 \right]
    \end{equation}
    otherwise, 
where $\bar{q_i} = \sigma_i(q_i,\lab_i(s_i))$,\, $i\in\{1,2\}$. The expectation is taken over the one-step transition probability matrices of the product MDP under the policies $(\mu_1, \mu_2)$.
The optimal policies that maximize $\mathbb{P}(\phi_1\wedge \phi_2)$ are obtained by maximizing the right-hand side of \eqref{eq:dp} with respect to $(\mu_1, \mu_2)$. This will give policies $\mu_1(s_1, s_2, q_1, q_2)$ and $\mu_2(s_1, s_2, q_1, q_2)$ that belong to $\Pi_f$ and require full state information. Since local MDPs have only access to limited observations and will use policies from $\Pi_l$, we discuss in the following how such policies can be designed using the framework of stochastic games.

Similar to the value functions $v_n$ in \eqref{eq:dp} for the specification $\phi_1\wedge\phi_2$, let us define local value functions $v^1_n$ and $v^2_n$ as
 $v^1_n(s_1, q_1, q_2)= v^2_n(s_2, q_1, q_2) = 1$
for all $q_1\in F_1, q_2\in F_2, s_1\in S_1, s_2\in S_2$ and $n\in\mathbb N_0$,
and
\begin{align}
& v^1_{n+1}(s_1, q_1, q_2)=\label{eq:v1}\\ 
      &\,\min_{\bar\ell\in \lab_2^{+}(q_2)}\max_{a_1}\min_{s_2\in \lab_2^{-1}(\bar\ell)}
    \sum_{\bar s_1}v^1_n(\bar s_1, \bar q_1, \bar q_2)\mathbf P_1(s_1,a_1,s_2, \bar s_1),\nonumber\\
    & v^2_{n+1}(s_2, q_1, q_2) =\label{eq:v2}\\ & \,\min_{\bar\ell\in \lab_1^{+}(q_1)}\max_{a_2}\min_{s_1\in \lab_1^{-1}(\bar\ell)}
    \sum_{\bar s_2}v^2_n(\bar s_2, \bar q_1, \bar q_2)\mathbf P_2(s_2, a_2,s_1,\bar s_2),\nonumber
\end{align}
where $\lab_i^{+}(q_i)$ is the set of all one step reachable labels $\bar \ell$ from current state of the automaton $q_i$, $\lab_i^{+}(q_i) := \{\bar\ell\,|\, \sigma_i(q_i, \bar \ell)\neq\emptyset\}$, and $v^1_0(s_1, q_1, q_2) = v^2_0(s_2, q_1, q_2) = 0$
for all $(q_1,q_2)\notin F_1\times F_2$, 
with $\bar q_i=\sigma_i(q_i, \lab_i(s_i))$,\,$i\in\{1,2\}$.
The value function $v^i_n$ depends on the local state $(s_i,q_i)$ and the information received from the neighboring component encoded in $q_{3-i}$, $i\in\{1,2\}$. Intuitively, these value functions give a lower bound on the probability of satisfying the local specification within the time horizon $n$.

We next show that the computation of value functions in \eqref{eq:dp} can be decomposed into the computation of local value functions $v^1_n$ and $v^2_n$. The product of these local value functions gives a lower bound for the original value functions of \eqref{eq:dp}. Each local value function gives a local policy that produces local actions only based on the limited observations available locally.

 \begin{theorem}
 \label{thm:lower_prob}
The value functions $v_n,v_n^1,v_n^2$ defined in \eqref{eq:dp},\eqref{eq:v1},\eqref{eq:v2} satisfy the inequality
\begin{equation}
\label{eq:value_function}
    v_n(s_1, s_2, q_1, q_2)\geq v^1_n(s_1, q_1, q_2)v^2_n(s_2, q_1, q_2),
\end{equation}
for all $n\in\mathbb N_0$ and all $s_1, s_2, q_1, q_2$.
Moreover, the policy $(\mu_{i0},\mu_{i1},\mu_{i2},\ldots)\in\Pi_l$ computed as $\mu_{in}:S_i\times Q^i\times Q^{3-i}\rightarrow A_i$, $i\in\{1,2\}$ that maximizes the right-hand sides of \eqref{eq:v1}--\eqref{eq:v2} will generate the obtained lower bound for the value of the global game $v_n$:
\begin{equation}
\mathbb{P}^{\mu_1, \mu_2}(\phi_1\wedge \phi_2)\ge  \lim_{n\to \infty} v_n^1\times  \lim_{n\to \infty} v_n^2.
\end{equation}
\end{theorem}

\subsection{Assume-Guarantee Interpretation}
\begin{theorem}
\label{thm:AGI}
    Let $\lozenge^n F$ denote reachability to a set $F$ within finite time bound $n$, and suppose policies $\mu_1,\mu_2$ are given, and according to Assumption~\ref{ass:obs} there are runs of the systems $(s_1(0), \ldots, s_1(n))$ and $(s_2(0), \ldots, s_2(n))$ that satisfy respectively $\lozenge^n F_1$ and $\lozenge^n F_2$. Then, 
\begin{align}
\mathbb{P}&^{\mu_1,\mu_2}(\lozenge^n(F_1\wedge F_2))\geq\nonumber\\
&\inf_{(s_2(0), \ldots, s_2(n))\in\Gamma_2}\mathbb{P}^{\mu_1}(\lozenge^n F_1\,|\, s_2(0), \ldots, s_2(n))\nonumber\\
 &\times\inf_{(s_1(0), \ldots, s_1(n))\in\Gamma_1}\mathbb{P}^{\mu_2}(\lozenge^n F_2\,|\, s_1(0), \ldots, s_1(n)),\label{eq:prob_lower}
\end{align}
where the infimum is taken over the set of satisfying runs:
$$\Gamma_i \!:=\!\left\{(s_i(0), \ldots, s_i(n))\,|\, (\lab_i(s_i(0),\ldots,\lab_i(s_i(n))\!\models\!\lozenge^n F_i\right\}.$$
\end{theorem}
\begin{proof}
When the policies are given, the definition of $v^1$ reduces to 
\begin{equation*}
\begin{array}{l}
v^1_{n+1}(s_1, q_1, q_2)=\\
\\
\left\{
\begin{array}{ll}
     1& \text{if } (q_1,q_2)\in F_1\times F_2\\
      \min_{\bar\ell\in \lab_2^{+}(q_2)
      }\min_{s_2\in \lab_2^{-1}(\bar \ell)}\left[\right.&\\
     \quad\sum_{\bar s_1}\left(\right.v^1_n(\bar s_1, \bar q_1, \bar q_2)\cdot&\\
     \qquad\mathbf P_1^{\mu_1}(s_1,a_1,s_2, \bar s_1)\left.\right)\left.\right]& \text{if } (q_1,q_2)\notin F_1\times F_2,
\end{array}
\right.
\end{array}
\end{equation*}
This recursive definition is exactly the computation of $v^1_n$ defined as
    $v^1_n(s_1, q_1, q_2) = 0$ if there is no sequence $(s_2(0), \ldots, s_2(n))$ satisfying $\lozenge^n F_2$ when the automaton $\DFA_2$ is initialized at $q_2$,
    Otherwise, $
v^1_n(s_1, q_1, q_2)=
\inf_{\Gamma_2} \mathbb{P}^{\mu_1}(\lozenge^n F_1\,|\, s_2(0), \ldots, s_2(n))
$.
This concludes the proof considering the fact that $\mathbb{P}^{\mu_1,\mu_2} (\lozenge^n(F_1\wedge F_2)) = v_n(s_1,s_2, q_{10}, q_{20})$ and combining it with inequality \eqref{eq:value_function}.
\end{proof}
\begin{corollary}
By taking the limit $n\rightarrow\infty$ from the inequality \eqref{eq:prob_lower} and then optimizing with respect to policies, we get the following result under Assumption~\ref{ass:obs}:
\begin{align}
&\sup_{(\mu_1,\mu_2)\in\Pi_l}\!\!\!\mathbb{P}^{\mu_1,\mu_2} (\phi_1\wedge\phi_2)
 \geq\nonumber\\
&\,\quad\sup_{\mu_1\in\Pi_l}\inf_{(s_2(0), s_2(1),\ldots)\models \phi_2}\mathbb{P}^{\mu_1}(\phi_1\,|\, s_2(0), s_2(1),\ldots)\nonumber\\
&         \times\sup_{\mu_2\in\Pi_l}\inf_{(s_1(0),s_1(1),\ldots)\models \phi_1}\mathbb{P}^{\mu_2}(\phi_2\,|\, s_1(0),s_1(1),\ldots).\label{eq:lower_bound_inf}
\end{align}
\end{corollary}

\subsection{The Case of No Communication}
When there is no communication between the individual components, the policy $(\mu_1,\mu_2)$ will be in the set $\Pi_n$: each component takes action using only the information of its own state and considering the worst case behavior of the other components. 
 Since there is no interaction between the subsystems, the optimization now targets the entire set of labels $\alphabet^2$ instead of just the labels one-step reachable on the automaton state $\bar\ell\in \lab_2^{+}(q_2)$.
The satisfaction probability computed under no-communication will be lower than the probability under limited communication, as formalized next.
\begin{theorem}
We have that %
\begin{align}
    \MoveEqLeft\sup_{(\mu_1,\mu_2)\in\Pi_n}\mathbb{P}^{\mu_1,\mu_2} (\phi_1\wedge\phi_2)
    \geq\nonumber\\
    &\left[
\sup_{\mu_1}\inf_{s_2\in S_2}\mathbb{P}^{\mu_1}(\phi_1)\right]
         \cdot\left[\sup_{\mu_2}\inf_{s_1\in S_1}\mathbb{P}^{\mu_2}(\phi_2)\right],
         \label{eq:lower_bound_inf2}
\end{align}
where the first term in the right-hand side is computed fully on the $G_1$ by assuming that the state of $G_2$ can be anywhere in its state space (similarly for the second term). 
Moreover, the lower bound in \eqref{eq:lower_bound_inf} improves the one in \eqref{eq:lower_bound_inf2}. 
\end{theorem}
\subsection{Putting it All Together}
Theorems~\ref{thm:lower_prob} and \ref{thm:AGI} provide a lower bound for the complete system based on values computed for the individual components via~\eqref{eq:v1} and \eqref{eq:v2}. Each of these equations represents a game (see Fig.~\ref{fig:my_label2}) where the minimizing player first selects a label, the maximizing player then selects an action, and the minimizing player finally resolves the state of their component. 
This game can be solved via RL and the policy extracted for the corresponding component is the one used by the maximizing player.

\section{Case Studies}

We provide a range of case studies to showcase the capability of the assume-guarantee RL approach
\footnote{The implementation is available at \url{https://doi.org/10.5281/zenodo.10377136}}. 
We employ Minimax-Q Learning~\citep{Littma96} (detailed in Appendix) 
and its deep learning extension to compute policies for each of the components that maximizes the lower bound on the global satisfaction of the objective.
The computations are performed on a laptop with Intel(R) Core(TM) i7-8650U CPU @ 1.90GHz, and 16GB RAM.

\subsection{Multi-Agent Grid World}
Consider a $4\times 3$ grid with two agents located at the center-top and center-bottom of the grid. The task for agents is to swap positions while avoiding each other. Due to the deterministic nature of the agent movements, the probability of satisfying such a specification is expected to be one, and the task of the learning is to find a policy that satisfies such a specification. Note that collision avoidance requires some form of communication between agents. We decompose the specification into two assume-guarantee specifications
\begin{equation*}
    \phi_1 = (\mathsf A_2 \Rightarrow \lozenge \mathsf B_1) \quad \text{ and }\quad
    \phi_2 = (\mathsf A_1 \Rightarrow \lozenge \mathsf B_2),
\end{equation*}
where $\mathsf A_i$ indicates the knowledge of the path of agent $i$, and $\mathsf B_i$ is its target location, for $i\in\{1,2\}$.
The optimal policy for the first agent is to go left, then up for three time steps, and then turn right (Red arrow in Fig.~\ref{fig:grid-reward} \textbf{Left}). The other agent's optimal policy is to go right, then down for three time steps, and then turn left (blue arrow in Fig.~\ref{fig:grid-reward} \textbf{Left}).

Using both minimax-Q learning and deep minimax-Q learning, we successfully learn the optimal policies of agents.
Note that since the two agents have the same dynamics, and the grid-world environment and the specifications and symmetric, we are able to learn the policy of one agent and use this symmetry to get the policy of the other agent.
Fig.~\ref{fig:grid-reward} (\textbf{Right}) shows the satisfaction probability as a function of training step. In training steps $i\times 10^5$, $i\in\{0,1,2,3,4,5\}$, we apply to the agents the policies learned up to that training step, and estimate the satisfaction probability empirically using Monte-Carlo simulations.
The learning procedure takes $10$ minutes.
Since both the agent and the policies are deterministic in this case study, the satisfaction probability is either zero or one.
The initial random policy does not satisfy the specification and the learning finds the optimal policy. 

\begin{figure}
    \centering
    \begin{tikzpicture}[scale=0.85]
	\draw[step=1cm,gray] (0,0) grid (3, 4);
	\draw[ultra thick, ->, >=stealth, draw=blue!70!white] (1.6,3.5) -- (2.5,3.5) -- (2.5,0.5) -- (1.6, .5);
	\draw[ultra thick, ->, >=stealth, draw=red!70!white] (1.4,.5) -- (.5,.5) -- (.5,3.5) -- (1.6, 3.5);
	\draw[thick] (0,0) rectangle (3,4);
        \path[thick] (0,-0.5) rectangle (3,-0.5);
\end{tikzpicture}
 \hspace{0.14cm}
    \includegraphics[scale=0.28]{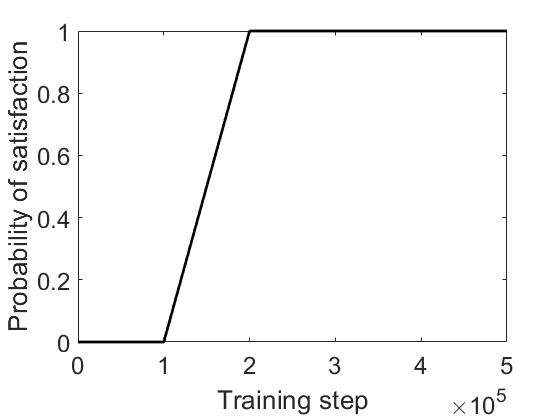}
    \caption{Multi-Agent Grid World. \textbf{Left:} Learned (optimal) policy for both agents. \textbf{Right:} The probability of satisfaction computed with policies learned in each training step.
    }
    \vspace{-1em}
    \label{fig:grid-reward}
\end{figure}
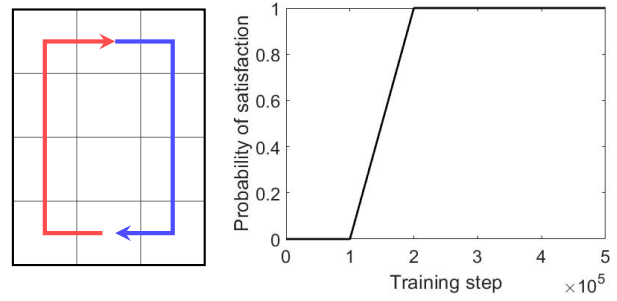

\subsection{Room Temperature Control}
Consider a network of rooms that are in a circular topology as in Fig.~\ref{fig:rooms}. Each room has a heater and its temperature is affected by the temperature of two adjacent rooms.  The evolution of temperature adapted from the work by~\citep{meyer}, can be described as
\begin{align}
    T^{j}_{k+1} &= T^{j}_{k} + \alpha(T^{j+1}_k + T^{j-1}_k - 2T^{j}_k)\nonumber\\
    &\quad\, + \beta(T^e - T^{j}_k ) + \gamma(T^{h} - T^{j}_k)u_k^{j},
\end{align}
where $T^{j}_{k}$ is the $j^{\text{th}}$ room temperature at time step $k$,  $T^e = -1$ is the outside temperature, $T^h = 50$ is the heater temperature, $u_k^{j}\in [0, 1]$ is the control input. The conduction factors are $\alpha = 0.45$,
$\beta = 0.045$, and $\gamma = 0.09$.  
The local specification for the $j^{\text{th}}$ room is
\begin{equation*}
    \phi_j = \square^n (T^j\in [17, 23]) \wedge \lozenge^n (T^j\in [21, 22]),
\end{equation*}
with $n=40$ time steps.
Note that the dynamics of the agents, the topology of the network, and the local specifications are symmetric. Therefore, we train once with three agents to maximize the probability $\mathbb P(\phi_j)$ with the assumption that rooms $(j-1)$ and $(j+1)$ satisfy respectively $\phi_{j-1}$ and $\phi_{j+1}$.
We use deep minimax-Q learning to learn the policies, which took $10$ minutes.
Fig.~\ref{fig:rooms} shows the empirical satisfaction probability (with a $95\%$ confidence interval) as a function of training step by running $1000$ Monte-Carlo simulations under the policy learned at each training step. 

\begin{figure}
    \centering
      \begin{tikzpicture}[>=latex,font=\sffamily,semithick,scale=1.75]
        \fill [blue!25] (0,0) -- (67.5:1) arc [end angle=0, start angle=67.5, radius=1] -- cycle;
        \fill [yellow!25] (0,0) -- (45:1) arc [end angle=0, start angle=45, radius=1] -- cycle;
        \fill [green!25] (0,0) -- (22.5:1) arc [end angle=0, start angle=22.5, radius=1] -- cycle;
        \draw [thick] (0,0) circle (1);
        \foreach \angle in {90,67.5,...,-67.5}
            \draw (\angle:1) -- (\angle-180:1);
        \node [circle,thick,fill=white,draw=black,align=center,minimum size=3cm] at (0,0) {};
    \end{tikzpicture}
    \includegraphics[scale=0.35]{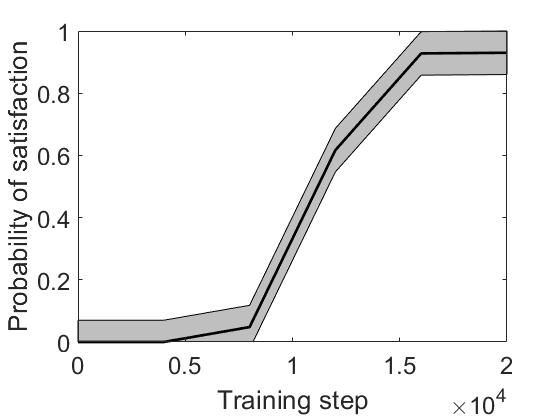}
    \caption{
    Room Temperature Control. {\bf Top:} Sixteen rooms in a circular topology.
    {\bf Bottom:}
    The probability of satisfaction computed with policies learned in each training step. 
}
    \label{fig:rooms}
    \vspace{-0.5cm}
\end{figure}
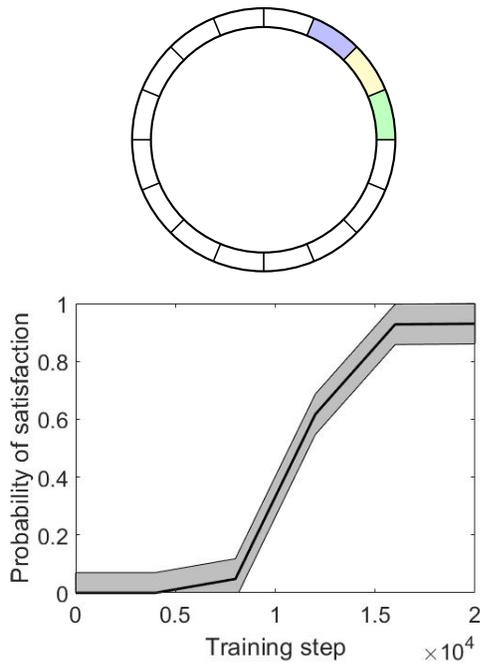

\subsection{Traffic Lights Control}
Consider the traffic intersection signal control problem for the kind of networks shown in Fig.~\ref{fig:traffic-schm}
with $n^2$ intersections each with a traffic light. 
The goal is to design policies to change the colors of the traffic lights.
The desired property is to keep the number of cars behind each traffic light at most $20$ in the next $100$ time steps: $\phi = \wedge_{i=1}^{n^2} \phi_i$ with
\begin{equation*}
    \phi_i \!=\! \square^{100} (\text{cars behind the $i^{th}$ intersection is at most $20$}).
\end{equation*}
A centralized approach for controlling the traffic light is not scalable for large traffic networks.
We decompose the network into $n^2$ intersections as agents with the properties $\phi_i$. The goal of each agent is to maximize $\mathbb P (\phi_i)$ under the assumption that the neighboring intersections satisfy their local specifications. The information communicated to each agent from other neighboring agents is the following: number of cars behind the traffic light is less than $10$ cars, between $10$ and $20$ cars, or more than $20$ cars.

We utilize the symmetric properties of the traffic network and use deep minimax-Q learning.
The learning procedure takes $100$ minutes.
Fig.~\ref{fig:traffic_prob}
shows the empirical satisfaction probability (with $95\%$ confidence interval) as a function of training step by running $1000$ Monte-Carlo Simulations under the policy learned in each training step. 

\begin{figure}
\vspace{-2em}
    \centering
    \includegraphics[scale=0.35]{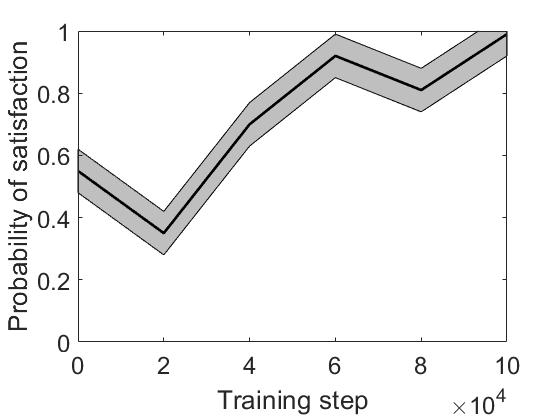}
    \caption{Traffic Lights Control (with nine intersections).
    The probability of satisfaction computed with policies learned in each training step. 
    }
    \label{fig:traffic_prob}
\end{figure}

\section{Related Work}

To address the complexity of large scale systems one approach is to decompose the system into subsystems or the general task into sub-tasks. There is a large body of literature about decomposing tasks into sub-tasks \citep{pnueli1985transition, benveniste2018contracts, chen2019compositional, eqtami2019quantitative}. Given a large-scale system decomposed into subsystems with their own sub-tasks, the question is to synthesize controllers for these subsystems to satisfy the global specification. Assume-guarantee reasoning intends to design controllers for the subsystems assuming the environment behaves in a certain way (the assumption). We, then, need to make sure the environment actually satisfies the assumption. 
The design of local policies has been studied in several disciplines \citep{amato2015scalable, guestrin2001multiagent, matignon2012independent, hammond2021multi, jothimurugan2021compositional, abate2021rational, yang2005survey, vinyals2019grandmaster}. In \emph{machine learning}, this problem is considered as the policy synthesis for partially observable MDPs; in \emph{control theory}, it takes the form of control synthesis for decentralized systems; while in \emph{game theory} it is treated as stochastic games with imperfect information.

Most of the research on assume-guarantee reasoning require knowing the model of the systems~\cite{henzinger1998you,kwiatkowska2010assume}. Model-based learning using the $L^\star$ algorithm is proposed by Angluin~\citet{fisman2018inferring}. 
The works \cite{bobaru2008automated, cobleigh2003learning, gheorghiu2008automated} utilize the $L^\star$ algorithm for assume-guarantee reasoning using abstraction-refinement, where counter-examples are generated to learn policies or verification.

There is also a large body of literature on decentralized control of dynamical systems, in which controllers are designed locally~\citep{malikopoulos2018decentralized, bakule2008decentralized, siljak2011decentralized, lavaei2019compositional, zhang2016assume}. The goal is to find a controller with the same performance as provided by a centralized controller. 
The assume-guarantee reasoning follows the same goal as persistence monitoring problem. One can consider the automaton associated with every subsystem as an event-trigger system. A controller is designed by just knowing the events of other subsystems. 

Another paradigm is to model the system as partially observable Markov decision processes (POMDPs) or partially observable Markov games \citep{hansen2004dynamic,brown2018superhuman}. For collaborative agents, the state labels of other subsystems are imperfect information available for the policy synthesis. However, the problem of finding optimal policies for POMDPs is PSPACE-hard~\cite{mundhenk2000complexity}.

\section{Conclusion}
We studied the problem of satisfying temporal properties modeled by deterministic finite automata (DFA) on systems with inherent structures.
A number of well-establish formal languages, including linear temporal logic (LTL) on finitary traces, co-safe LTL, and LTL with only past operators, can be compiled into DFA.
We provided a modular data-driven approach that uses local RL on assume-guarantee contracts and used the dynamic programming characterization of the solution to prove a lower bound for the satisfaction probability of the global specification. The agent utilizes the
information it receives through communication and observing transitions in a coarse abstraction of neighboring systems.
In the future, we plan to go beyond specifications captured by DFA and extend our approach to full LTL specifications using approximation methods with formal guarantees. 

\newpage
\section{Acknowledgments}
M. Kazemi's research is funded in part by grant EP/W014785/1. The research of S. Soudjani is supported by the following grants: EPSRC EP/V043676/1, EIC 101070802, and ERC 101089047. This work was supported in part by the NSF through grant CCF-2009022 and the NSF CAREER award CCF-2146563.
\bibliography{ref}
\appendix
\onecolumn

\section{From Concurrent to Turn-based Games}
\label{sec:concurrent}
We are interested in designing deterministic (non-stochastic) control for a network of MDPs.
Since concurrent stochastic games are not determined in deterministic  (non-stochastic) policies~\citep{de2007concurrent}, we consider turn-based approximation of the games where the choices of the controller are visible to the environment. 
This assumption is natural as often the controller needs to be physically implemented and the environment may have access to or may learn such implementation of the controller and exploit this knowledge adversarially. 
The control policies on the resulting turn-based games are more pessimistic and provide a lower bound on the value of the game.

For a concurrent Markov game $G = (S,s_0, A, B, \mathbf P, \lab)$, we define a turn-based game $G^{\bullet} = (S^{\bullet}, s_0, A^{\bullet}, \mathbf P^{\bullet}, \lab^{\bullet})$ over an extended state space $S^{\bullet} = S \cup (S \times A)$ and action space $A^{\bullet} = A \cup B$ with transition probabilities $\mathbf P^{\bullet}: S^{\bullet} \times A^{\bullet} \times S^{\bullet} \to [0,1]$. 
The state space $S^{\bullet}$ is  partitioned into set $S^{\bullet}_1 = S$ of states controlled by Player~1 (controller states) and the set $S^{\bullet}_2 = (S \times A)$ of states controlled by Player~2 (environment states). 
Whenever the system is in a state $s\in S^{\bullet}_i$, for $i\in\{1,2\}$, Player $i$ is allowed to take an action. The transition function is such that 
\begin{eqnarray*}
\mathbf P^{\bullet}(s, c)(s') = 
\begin{cases}
1 & \text{if } s \in S^{\bullet}_1 \text{ and } s' = (s, c) \in S^{\bullet}_2\\
\mathbf P (s, a, c) & \text{ if  } s = (p, a) \in S^{\bullet}_2  \\
0 & \text{ otherwise.}
\end{cases}
\end{eqnarray*}

\begin{theorem}
For a concurrent Markov game $G = (S,s_0, A, B, \mathbf P, \lab)$ and its corresponding turn-based game $G^{\bullet}$,  the value of the game in $G^{\bullet}$ is equal to the lower value of the concurrent game. In other words, for every policy of Player 1 in $G^{\bullet}$, there exists a policy of Player 1 in $G$, that provides the same value.
\end{theorem}

\section{Minimax-Q Learning Algorithm}
\label{sec:minimax}
Consider a reward function $\Rwd: S {\times} A {\times} S \to \Real$ for
a Markov game $G$.
From initial state $s_0$, the game evolves by having the player that controls $s(n)$ at time step $n$ select an action $a(n+1) \in A(s(n))$. The state then evolves under distribution $\mathbf P(s(n),a(n),\cdot)$ resulting in a next state $s(n+1)$ and reward $r_{n+1} = \Rwd(s(n), a(n), s(n{+}1))$.
Assume that the state space is partitioned as $S = S_\mMAX \cup S_\mMIN$ into the space of a Max player $S_\mMAX$ and the space of a Min player $S_\mMIN$. 
Given a discount factor $\gamma \in [0,1)$, the payoff (from player Min to player Max) is the $\gamma$-discounted sum of rewards, \emph{i.e.,}
$\sum_{n=0}^\infty r_{n+1}\gamma^n$.
The objective of player Max is to maximize the expected payoff, while the objective of player Min is the opposite.
	
A policy is a mapping from states to actions. A policy $\rho_* \in \Pi_{\mMAX}$ is optimal if it maximizes
\begin{equation*}
	\inf_{\xi \in \Pi_{\mMIN}} \EE^s_{\rho, \xi} \Bigl[\sum_{n=0}^\infty \Rwd(s(n), a(n), s(n+1)) \gamma^n\Bigr],
	\end{equation*}
	which is sum of rewards under the worst policy of player Min.
	The optimal policies for player Min are defined analogously.
	The goal of RL is to compute optimal policies for both players with samples from the game, without \emph{a priori}
	knowledge of the transition probability and rewards.
	The RL solves this by learning a state-action value function, called $Q$-values, defined as 
	\[Q_{\rho,\xi}(s,a) = \EE^{s,a}_{\rho,\xi} \Bigl[\sum_{n=0}^\infty \Rwd(s(n), a(n), s(n+1)) \gamma^n\Bigr],
	\]
	where $\rho \in \Pi_{\mMAX}$ and $\xi \in \Pi_\mMIN$.
	Let  
	$Q_*(s,a) = \sup_{\rho \in \Pi_\mMAX} \inf_{\xi \in \Pi_\mMIN}  Q_{\rho,\xi}(s, a)$ be the optimal value.
	Given $Q_*(s,a)$, one can extract the policy for both players by
	selecting the maximum value action in states controlled by player Max and the
	minimum value action in states controlled by player Min. 
	The following Bellman optimality equations characterize the optimal solutions and forms the basis for computing the Q-values by dynamic programming:
	\[
	Q_*(s, a) {=} \sum_{s' \in S}\mathbf P(s,a,s') {\cdot} \Bigl( r(s, a, s') {+} \gamma {\cdot} \hspace{-1em}\opt_{a'\in A(s')} Q_*(s',a')\Bigr),
	\]
	where $\opt$ is $\max$ if $s' \in S_\mMAX$ (the state space of the player Max) and $\min$ if $s' \in S_\mMIN$ (the state space of the Min player).
	Minimax-Q
	learning~\citep{Littma94} estimates the dynamic programming update from the
	stream of samples by performing the following update at each time step: 
\[
	Q(s(n),a(n)) :=  (1-\alpha_n) Q(s(n), a(n)) +\alpha_k(r_{n+1} + \gamma \!\!\opt_{a' \in A(s(n+1))}\!\! Q(s(n+1),a')),
\]
	where $\alpha_n \in (0,1)$, a hyperparameter, is the learning rate at time step $n$.
 The Minimax-Q algorithm produces the
	controller directly, without producing estimates of the unknown
	system dynamics: it is \emph{model-free}. 
	Moreover, it reduces to classical Q-learning~\citep{watkins1989learning} for MDPs, \emph{i.e.,} when $S_{\mMIN} = \emptyset$.
	
	\begin{theorem}[Minimax-Q Learning~\citep{Littma96}]
		The minimax-Q learning algorithm converges to the unique fixpoint $Q_*(s,a)$ if $r_n$ is bounded, the learning rate satisfies the Robbins-Monro conditions, i.e., 
		$\sum_{n=0}^\infty \alpha_n = \infty \text{  and  } \sum_{n=0}^\infty \alpha_n^2 < \infty$,
		and all state-action pairs are seen infinitely often.
	\end{theorem}

\section{Proof of Theorem~\ref{thm:lower_prob}}
\begin{proof}
We apply induction on $n$ and use the dynamic programming equations \eqref{eq:dp}. For $n=0$ or when $(q_1,q_2)\in F_1\times F_2$, both sides of \eqref{eq:value_function} are the same. Suppose the inequality is true for $n$. We show that it also holds for $n+1$ and $(q_1,q_2)\notin F_1\times F_2$.
\begin{equation}
\begin{array}{ll}
    v_{n+1}(s_1,s_2, q_1, q_2) &= \mathbb E^{\mu_1,\mu_2}\left[ v_n(\bar s_1, \bar s_2, \bar q_1, \bar q_2)\,|\, s_1, s_2, q_1, q_2 \right]\\
    &\geq
    \mathbb E^{\mu_1,\mu_2}\left[ v_n^1(\bar s_1, \bar q_1, \bar q_2)v_n^2(\bar s_2, \bar q_1, \bar q_2)\,|\, s_1, s_2, q_1, q_2 \right]\\
     & \geq
     \sum_{\bar s_1}v^1_n(\bar s_1, \bar q_1, \bar q_2)\mathbf P_1(s_1,a_1,s_2, \bar s_1)\\
     &\qquad\qquad
     \cdot\sum_{\bar s_2}v^2_n(\bar s_2, \bar q_1, \bar q_2)\mathbf P_2(s_2, a_2,s_1,\bar s_2).
\end{array}
\end{equation}
While we get a product structure in the right-hand side, both terms in the product are functions of $s_1$ and $s_2$. In order to make the first term independent of $s_2$ and the second term independent of $s_1$, we compute the worst case with respect to these variables by first fixing their labels, obtaining the input, and then the worst case with respect to the label:
\begin{equation}
\begin{array}{ll}
     &v_{n+1}(s_1,s_2, q_1, q_2)\geq \\
     &\min_{\bar\ell\in \lab_2^{+}(q_2)}\max_{a_1}\min_{s_2\in \lab_2^{-1}(\bar\ell)}%
     \sum_{\bar s_1}v^1_n(\bar s_1, \bar q_1, \bar q_2)\mathbf P_1(s_1,a_1,s_2, \bar s_1)\\
     &
     \cdot \min_{\bar \ell\in\lab_1^{+}(q_1)}\max_{a_2}\min_{s_1\in \lab_1^{-1}(\bar\ell})\sum_{\bar s_2}v^2_n(\bar s_2, \bar q_1, \bar q_2)\mathbf P_2(s_2, a_2,s_1,\bar s_2)
    \\
    & = v^1_{n+1}(s_1, q_1, q_2) \cdot v^2_{n+1}(s_2, q_1, q_2).
\end{array}
\end{equation}
This will give us the policies
\begin{align*}
    & a_1 = \mu_1(s_1, \bar q_1, \bar q_2) = \mu_1(s_1, q_1, q_2, \lab_2(s_2)),\\
    & a_2 = \mu_2(s_2, \bar q_1, \bar q_2) = \mu_2(s_2, q_1, q_2, \lab_1(s_1)),
\end{align*}
 where $s_i$ is the current state of $G_i$, $\bar q_i$ is the updated mode of $\DFA_i$: $\bar q_i=\sigma_i(q_i, \lab_i(s_i))$, and $\lab_i^{+}(q_i)$ is the set of all one step reachable labels $\bar \ell$ from current state of the automaton $q_i$ where $\lab_i^{+}(q_i) := \{\bar\ell\,|\, \sigma_i(q_i, \bar \ell)\neq\emptyset\}$, \,$i\in\{1,2\}$ . Essentially, for the model $G_i$ we let the other model $G_2$ to take a transition in $\DFA_2$ and we let the policy of $G_1$ to depend on this particular transition. For the purpose of the dynamic programming, we have to take the worst case behavior with respect to such possible transitions of $G_2$. We get the similar structure for the policy of $G_2$.
\end{proof}

\section{Extension to MDPs with Continuous Spaces}
Although we presented our results for finite MDPs, the results also holds for MDPs with continuous (uncountable) state spaces \citep{SA13,lavaei2020ICCPS}. The dynamic programming formulation in \eqref{eq:dp}, Theorem~\ref{thm:lower_prob} and the assume-guarantee interpretation in Theorem~\ref{thm:AGI} remains the same. All the steps of the proof of statements remain true after replacing sums over finite states with integrals over continuous spaces.

\begin{figure}
    \centering
    \includegraphics[scale=0.3]{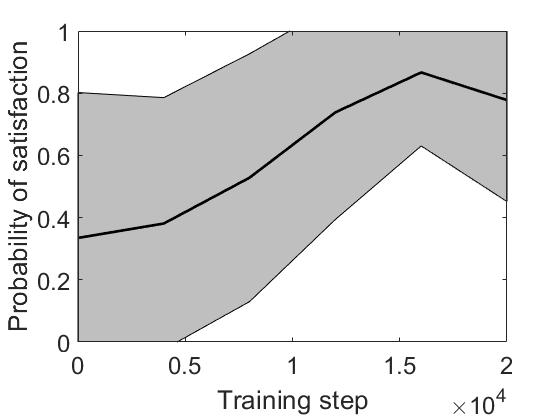}
    \includegraphics[scale=0.3]{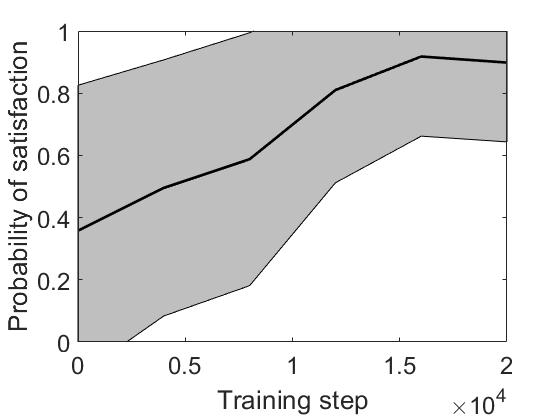}
    \caption{Room temperature control. Probability of satisfying assume-guarantee specifications for two representative rooms with $95\%$ confidence intervals.}
    \label{fig:Tem_local}
\end{figure}

\begin{figure}
    \centering
    \includegraphics[scale=0.3]{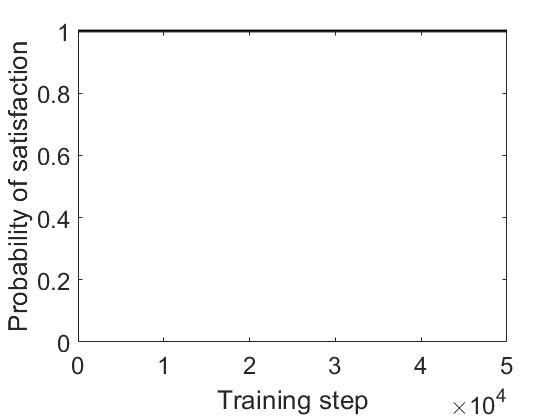}
    \includegraphics[scale=0.3]{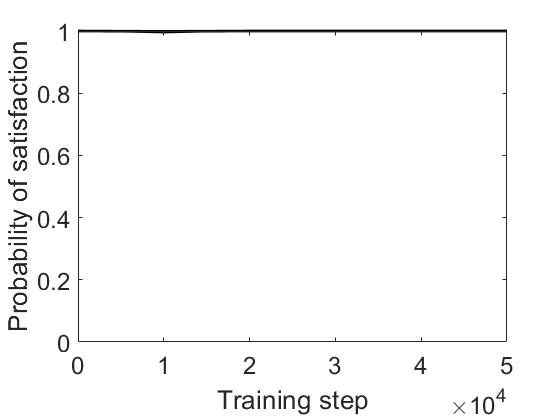}
    \includegraphics[scale=0.3]{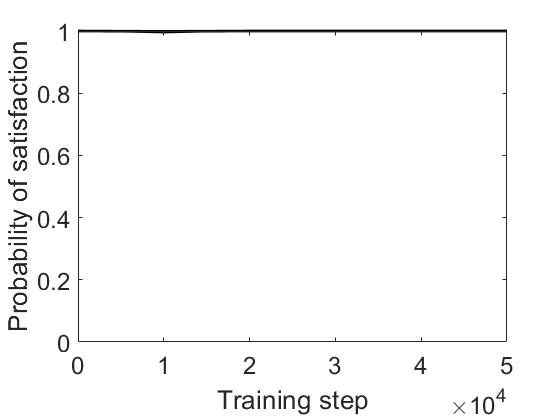}
    \includegraphics[scale=0.3]{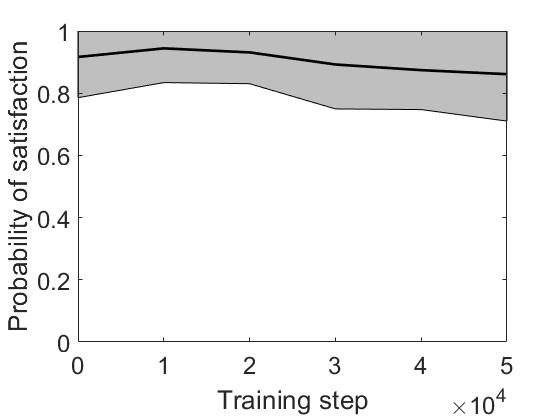}
    \caption{Traffic light control environment. Probability of satisfaction for a $3\times 3$ grid for each intersection traffic light.}
    \label{fig:Traffic_local_3}
\end{figure}

\begin{figure}
    \centering
    \includegraphics[scale=0.3]{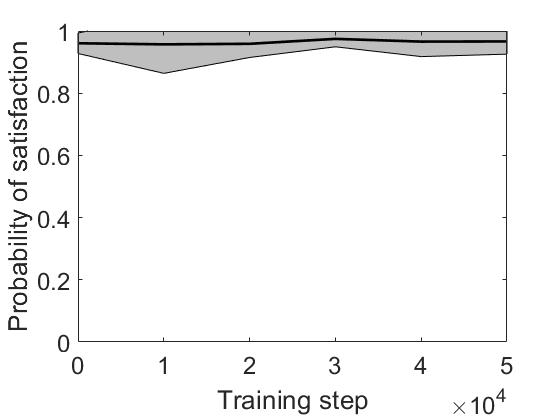}
    \includegraphics[scale=0.3]{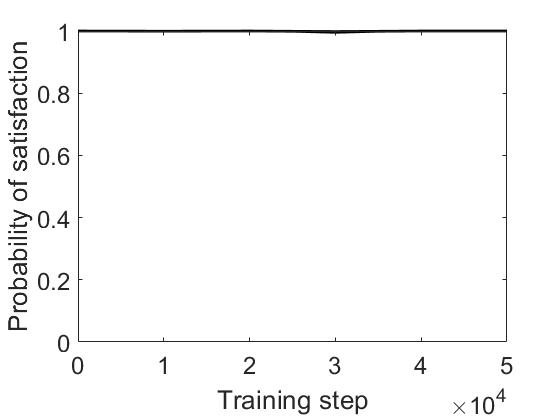}
    \includegraphics[scale=0.3]{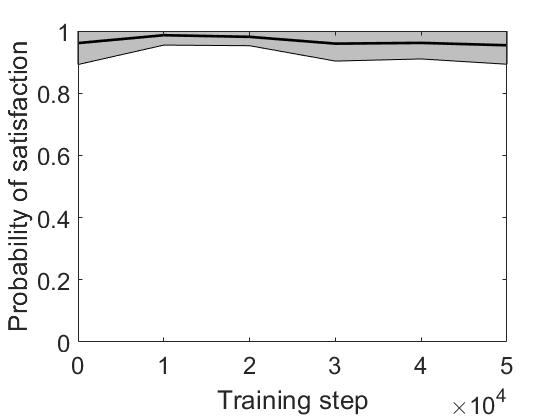}
    \includegraphics[scale=0.3]{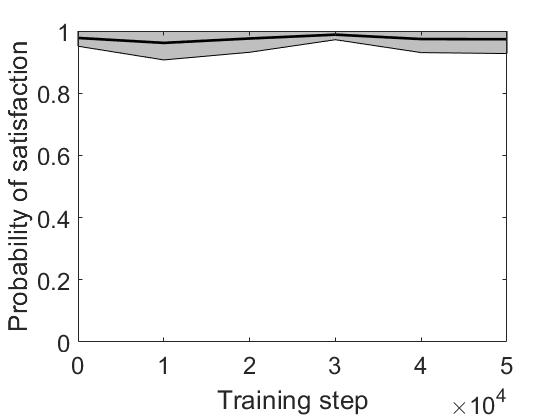}
    \includegraphics[scale=0.3]{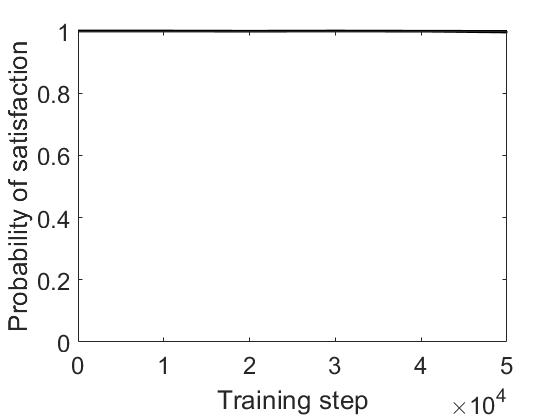}
    \caption{Traffic light control environment. Probability of satisfaction for a $3\times 3$ grid for each intersection traffic light.}
    \label{fig:Traffic_local_5}
\end{figure}

\section{Implementation using independent Q-learning}
We compare our approach to an existing approach in the literature. 
One way to design a policy is to consider the behavior of other agents as part of environment. In this case, we can train multiple agents independent of each other (e.g. Independent Q-learning (IQL)). The issue with this approach is that the environment in presence of other agents is not stationary anymore. In the following we provided the results of implementation of the case studies using IQL:
\begin{figure}
    \centering
    \includegraphics[scale=0.35]{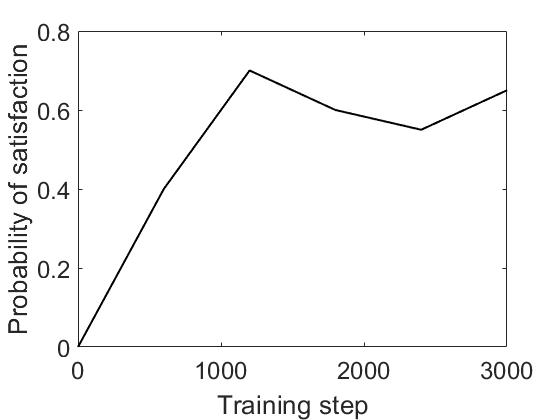}
    \caption{Multi-Agent Grid World. The plots show the probability of satisfaction for $20$ learning procedure.}
    \label{fig:Grid_world_reward}
\end{figure}

\begin{figure}
    \centering
    \includegraphics[scale=0.35]{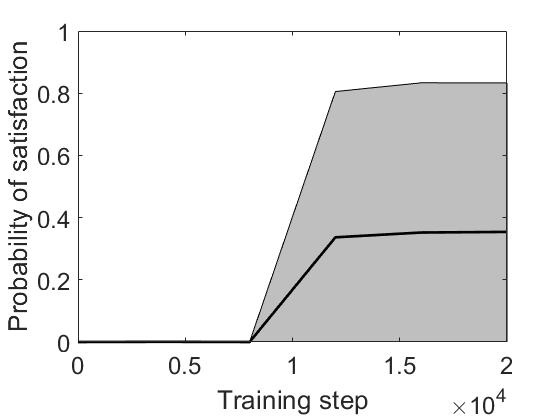}
    \includegraphics[scale=0.35]{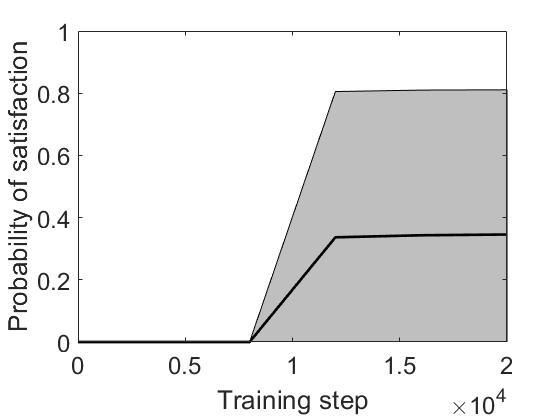}
    \caption{Room temperature control. \textbf{Left:} Probability of satisfaction for agent $1$ and \textbf{Right:} Probability of satisfaction for agent $2$.}
    \label{fig:Room_Reward}
\end{figure}

Fig.~\ref{fig:room-comp} compares the performance of our algorithm with the case of no communication in the room temperature control case study.
\begin{figure}
    \centering
    \includegraphics[scale=0.4]{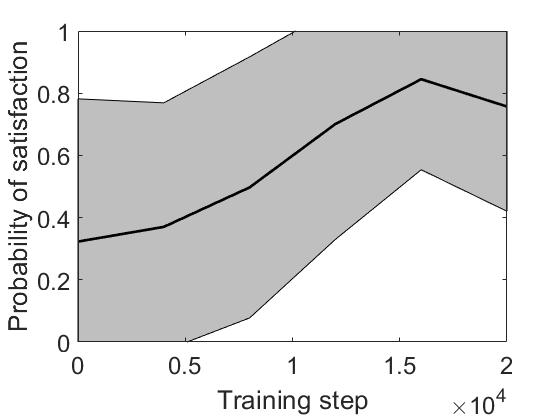}
    \includegraphics[scale=0.4]{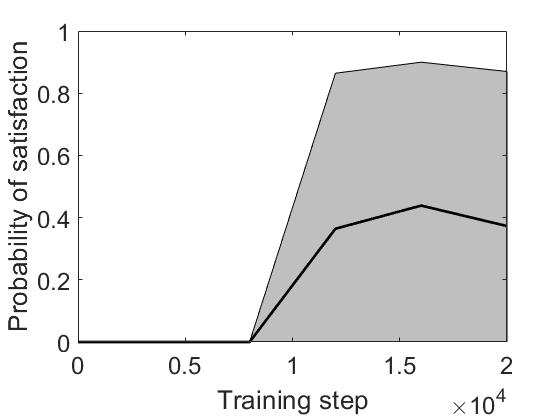}
    \caption{
    Room Temperature Control. 
    The probability of satisfaction computed with policies learned in each episode. The probability is estimated empirically with $95\%$ confidence bounds.
    The probability of satisfaction for our algorithm (\textbf{left}) and independent Q-learning (\textbf{right}) for the case of no communication between the agents.
    }
    \label{fig:room-comp}
\end{figure}

\end{document}